\documentclass[twoside,leqno,twocolumn]{article}
  \usepackage{ltexpprt}
   
  \usepackage{algorithm,algpseudocode}
  \usepackage{pgfplots,pgfplotstable}
  \usepgfplotslibrary{groupplots}
  \usepackage{amsmath,amssymb,upgreek}
  \usepackage{mathtools}
  \usepackage{environ}
  \usepackage{tabularx}
  \usepackage{enumerate}
  \pgfplotsset{compat=1.10}
  \usepackage{tikz}
  \usetikzlibrary{matrix,decorations.pathreplacing, calc, positioning}
  \usepackage[english]{babel}
  \usepackage[utf8x]{inputenc}
  \usepackage{nicefrac}
  \usepackage{filecontents}
  \usepackage{booktabs,multirow}
  \usepackage{cases} 
  \usepackage{hhline}
  \usepackage{bbm} 
  \usepackage{url}

\DeclareMathOperator{\prox}{prox}

\DeclareMathOperator{\tr}{tr}

\DeclareMathOperator{\diag}{diag}

\definecolor{TUGreen}{RGB}{132,184,24}
\definecolor{TUGray}{RGB}{104,104,104}
\definecolor{mygreen}{RGB}{181,221,109}
\definecolor{myyellow}{RGB}{255,237,111}
\definecolor{myseablue}{RGB}{147,213,198}
\definecolor{mylila}{RGB}{189,132,190}
\definecolor{myorange}{RGB}{246,179,98}
\definecolor{myred}{RGB}{247,129,116} 
\definecolor{myblue}{RGB}{130,176,208}
\definecolor{myminthe}{RGB}{208,235,189}
\definecolor{mypink}{RGB}{251,205,227}

\definecolor{ctrust}{RGB}{171,221,164}
\definecolor{cdens}{RGB}{43,131,186}
\definecolor{cPrimp}{RGB}{215,25,28}
\definecolor{cPan}{RGB}{253,174,97}
\makeatletter
\newcommand{\problemtitle}[1]{\gdef\@problemtitle{#1}}
\newcommand{\probleminput}[1]{\gdef\@probleminput{#1}}
\newcommand{\problemquestion}[1]{\gdef\@problemquestion{#1}}
\NewEnviron{problem}{
  \problemtitle{}\probleminput{}\problemquestion{}
  \BODY
  \par\addvspace{.5\baselineskip}
  \noindent
  \begin{tabularx}{\textwidth}{ p{2pc} p{16.5pc} c}
    \multicolumn{2}{l}{\@problemtitle} \\
    \textbf{Given} & \@probleminput \\
    \textbf{Find} & \@problemquestion
  \end{tabularx}
  \par\addvspace{.5\baselineskip}
}
\makeatother

\newlength{\cellwidth}
\newlength{\cellheight}
\let\END=\eof

\newcommand{\colormatrix}[1]
{\parbox{\CMcolumns\cellwidth}{%
  \baselineskip=\cellheight
  \lineskip=0pt
  \def\one{1}%
  \def\zero{0}%
  \bgroup
    \countdef\col=1
    \col=0
    \CMParse#1\END
  \egroup
}}

\def\CMParse#1{\ifx#1\END\else
  \if#1\one\relax{\color{\onecolor}\rule{\cellwidth}{\cellheight}}\fi
  \if#1\zero\relax{\color{\zerocolor}\rule{\cellwidth}{\cellheight}}\fi
  \advance\col by 1
  \ifnum\col<\CMcolumns\relax\else
    \hfil
    \col=0
  \fi
\expandafter\CMParse\fi}

\newcommand{\setcolormatrix}[6]
{\global\cellwidth=\dimexpr #1/#3\relax
 \global\cellheight=\dimexpr #2/#4\relax
 \global\def\CMcolumns{#3}%
 \global\def\onecolor{#5}%
 \global\def\zerocolor{#6}%
}

\makeatletter
\newcommand{\thickbar}{\mathpalette\@thickbar}
\newcommand{\@thickbar}[2]{{#1\mkern1.5mu\vbox{
  \sbox\z@{$#1\mkern-1.5mu#2\mkern-1.5mu$}%
  \sbox\tw@{$#1\overline{#2}$}%
  \dimen@=\dimexpr\ht\tw@-\ht\z@-.8\p@\relax
  \hrule\@height.8\p@ 
  \vskip\dimen@
  \box\z@}\mkern1.5mu}
}
\makeatother

\begin{document}

\title{\Large The Trustworthy Pal: Controlling the False Discovery Rate\\in Boolean Matrix Factorization
}
\author{Sibylle Hess\thanks{AI Group, TU Dortmund, first.last@tu-dortmund.de}\\
\and
Nico Piatkowski\addtocounter{footnote}{-1}\footnotemark\\
\and
Katharina Morik\addtocounter{footnote}{-1}\footnotemark}
\date{}

\maketitle


\fancyfoot[R]{\footnotesize{\textbf{Copyright \textcopyright\ 2018 by SIAM\\
Unauthorized reproduction of this article is prohibited}}}





\begin{abstract} \small\baselineskip=9pt
Boolean matrix factorization (BMF) is a popular and powerful technique for inferring knowledge from data. The mining result is the Boolean product of two matrices, approximating the input dataset. The Boolean  product is a disjunction of rank-1 binary matrices, each describing a feature-relation, called pattern, for a group of samples. Yet, there are no guarantees that any of the returned patterns do not actually arise from noise, i.e., are false discoveries. In this paper, we propose and discuss the usage of the false discovery rate in the unsupervised BMF setting. We prove two bounds on the probability that a found pattern is  constituted of random Bernoulli-distributed noise. Each bound exploits a specific property of the factorization which minimizes the approximation error---yielding new insights on the minimizers of Boolean matrix factorization. This leads to improved BMF algorithms by replacing heuristic rank selection techniques with a theoretically well-based approach. Our empirical demonstration shows that both bounds deliver excellent results in various practical settings. 
\end{abstract}
\section{Introduction}
Often enough in explorative data mining, the user is left alone with the result; a bunch of groupings which supposedly expresses the underlying relations in the dataset. The absence of quality guarantees is an eyesore for the painstaking data miner.
Whenever data is collected from an imperfect (noisy) channel---arising from tainted or inaccurate measurements, or transmission errors---the method of choice might be fooled by the noise, resulting in phantom patterns which actually don't exist in the data. Thus, the investigation of \emph{trustworthiness} of data mining techniques is important in practice. While some approaches for the supervised setting exist, e.g., significant pattern mining \cite{Lopez/etal/2015a}, statistical emerging pattern mining \cite{Komiyama/etal/2017a} and references therein, insights for the unsupervised case are still missing. 

It is not feasible to investigate all data mining methods at once. Boolean matrix factorization (BMF) is a popular and powerful technique for inferring knowledge from data. A factorization of a binary data matrix $D$ represents a product of two (or more) factor matrices $X$ and $Y$. When $D$ is an $m$-by-$n$ matrix, $X$ is $n$-by-$r$, and $Y$ is $m$-by-$r$. While $n$ and $m$ are given by the data, the quantity $r$---the \emph{rank} of the factorization---has to be specified. This quantity determines how many parts, called tiles, the data model comprises and its correct estimation is crucial for the quality of the result. An example decomposition of rank three is shown in Fig.~\ref{fig:decompexa}. Assuredly, this exemplary decomposition arises from raw binary image (bitmap) data. No explicit information about `color' was fed to the algorithm, the grouping of colors is the result of the factorization alone. 

\begin{figure}
\centering
\includegraphics[width=2cm]{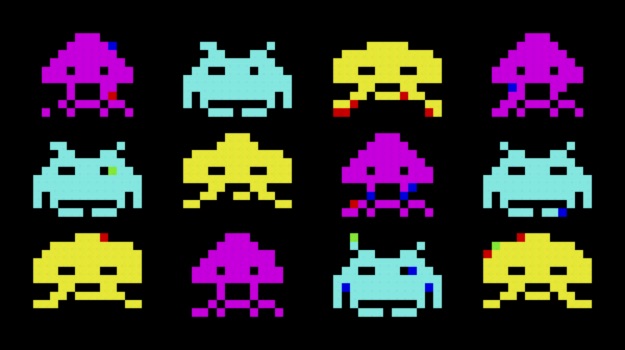}
\includegraphics[width=2cm]{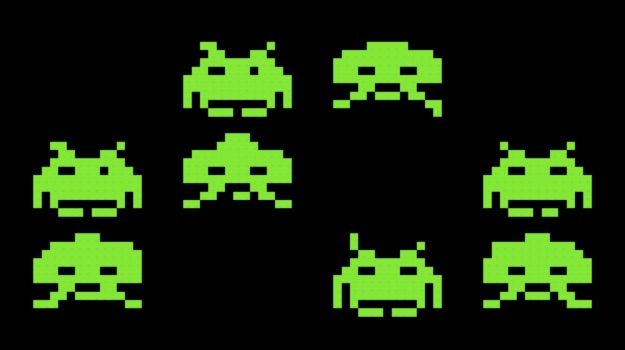}
\includegraphics[width=2cm]{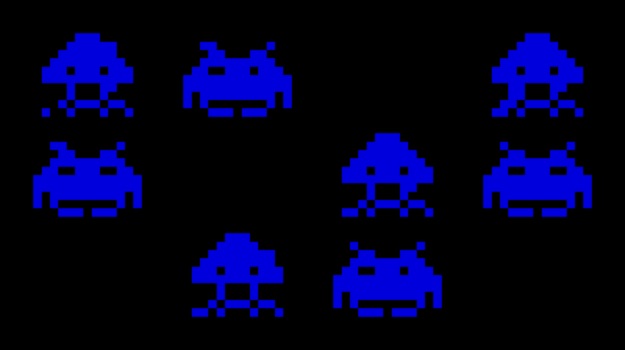}
\includegraphics[width=2cm]{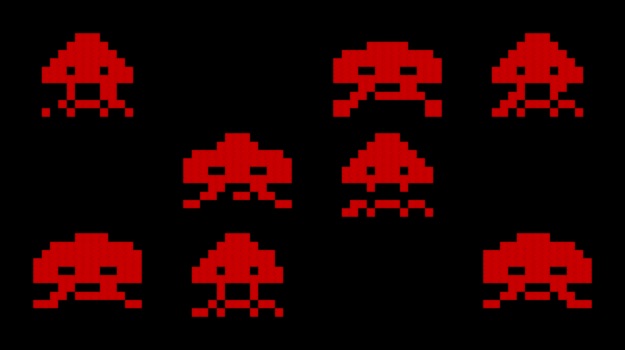}
\caption{Exemplary unsupervised decomposition of the input data (left) into three parts. In this example, the `parts' are the color channels. The data was represented via a block-wise binary encoding of pixel colors.\label{fig:decompexa}}
\end{figure}

Now, if some of the modeled tiles emerge from noise, false discoveries happen and the algorithm overfits.
We prove two bounds on the probability that a found tile is constituted of random Bernoulli-distributed noise. Both allow us to exploit specific properties of a tile, resulting in different strengths for different types of input data.
The bounds require an additional input from the user: an estimate to the noise level $p$. While this might seem to be a prohibitive burden, our experimental results show that a rough estimate suffices---the user should merely know if her data is pretty noisy or not so much.


State-of the art BMF techniques \cite{miettinen2011model,Hess/etal/2017a} employ rather complicated regularization terms to filter the structure from the noise and to correctly estimate the rank. The belief in the correctness of these algorithms is based on empirical evaluations, inter alia, experiments with synthetically added Bernoulli-distributed noise. 
Under this noise assumption, we employ our bounds to devise a new, theoretically well-founded, rank estimation strategy. Since our technique may be used as a plug-in replacement for existing heuristics, our findings improve a whole class of BMF algorithms.
Our main contributions are:
\begin{itemize}
\item the first provision of bounds on the probability that a tile with specified properties is generated from random noise;
\item the validation of required properties for factorizations which minimize the approximation error, showing that both bounds are non-trivial;
\item the exemplification of algorithmic use of the bounds for automatic rank selection and the empirical evaluation on synthetic and real-world datasets---a step towards trustworthy data mining. 
\end{itemize}
\section{Background}\label{sec:background}
We denote the input data by a binary matrix $D\in \{0,1\}^{m\times n}$. If the underlying data is not binary, we employ a problem-specific discretization technique. 

For any $m$-by-$n$ matrix $M$, 
we designate matrix norms as $\|M\|^2=\sum_{j=1}^m \sum_{i=1}^n M_{ji}^2$ for the squared Frobenius norm and $|M|=\sum_{j=1}^m \sum_{i=1}^n |M_{ji}|$ for the entrywise 1-norm. While both obey different analytic properties, their values are equivalent for binary matrices in the sense that $|M|=\|M\|^2$. The Frobenius inner product is defined for matrices $X$ and $Y$ as $\langle X,Y\rangle_F=\tr(X^\top Y)$. For nonnegative matrices $X$ and $Y$, the Frobenius inner product equates the entrywise 1-norm of the (component-wise) Hadamard product (denoted by $\circ$), i.e., $\langle X,Y\rangle_F=|X\circ Y|$. 

We write $\mathbf{1}$ to represent a matrix having all elements equal to 1. The dimension of such a matrix is always inferable from the context. $\thickbar{M}=\mathbf{1}-M$ is the complement of a binary matrix $M$. Here, the dimension of the matrix $\mathbf{1}$ is the dimension of $M$.

Given two binary matrices $X\in\{0,1\}^{n\times r}$ and $Y\in\{0,1\}^{m\times r}$, we denote their Boolean matrix product by $Y\odot X^\top$, where 
\begin{align*}
(Y\odot X^\top)_{ji} = \bigvee_{s=1}^rY_{js}X_{is}
\end{align*}
and $a \vee b$ denotes the logical \texttt{or} of truth-values $a$ and $b$.

\subsection{Boolean Matrix Factorization}
We assume that the data $D\in\{0,1\}^{m\times n}$ originates from a matrix product and some noise, i.e.,
\begin{align}\label{eq:BMF0}
D=(Y^*\odot {X^*}^\top \vee N_+) \wedge (\thickbar{Y^*\odot {X^*}^\top} \vee \thickbar{N_-})\;,
\end{align}
where $X^*\in\{0,1\}^{n\times r}$ and $Y^*\in\{0,1\}^{m\times r}$ are the underlying factor matrices  and $N_+,N_-\in\{0,1\}^{m\times n}$ are the binary positive and negative noise matrices. 

A Boolean product $Y\odot X^\top$ is the disjunction of $r$ matrices; each matrix is defined by the outer product $Y_{\cdot s}X_{\cdot s}^\top $ of the $s$-th column vectors. We refer to a pair of such column vectors $(X_{\cdot s},Y_{\cdot s})$ as a \emph{tile}. The left vector of a tile indicates a pattern, and the right vector tells us where this pattern appears in the data. Correspondingly, $X$ is called the pattern and $Y$ the usage matrix.

Given the data matrix $D$, we wish to extract the original relations denoted by $X^*$ and $Y^*$, which is unfortunately impossible for practical applications. Hence, surrogate tasks are considered in which we try to obtain good estimates of $X^*$ and $Y^*$. A straightforward way is to derive the factorization which minimizes the residual sum of absolute values $L(X,Y)=|D-Y\odot X^\top|$. Since this is still a hard problem unless $\mathbf{P} = \mathbf{NP}$~\cite{discreteBasisProb}, various heuristic algorithms have been developed to find approximate solutions. A rather recent branch of research allows us to employ numerical optimization methods instead~\cite{Hess/etal/2017a}. The corresponding framework is called \textsc{PAL-Tiling}, whose core is the minimization of a smooth version of the approximation error, namely $F(X,Y)=\frac{1}{2}\|D-YX^\top\|^2$, subject to binary constraints.
The objective $F$ is well-known from nonnegative and binary matrix factorization~\cite{wang2013nmfReview,zhangApplication}. 
We will revisit this surrogate in Sec.~\ref{sec:alg}, when we explain how practical algorithms can benefit from our bounds.

\section{Bounding False Discoveries}\label{sec:ctrlFDRBMF}
The first step towards trustworthy pattern mining is a measure of trustworthiness. 
The False Discovery Rate (FDR)~\cite{benjamini1995fdr} is a simple yet powerful way to express the probability that something goes wrong.
\begin{Definition}[FDR]
Given a finite set $\mathcal{H}$ of null hypotheses from which $r$ are rejected. Let $v$ denote the number of erroneously rejected null hypotheses. We say that \emph{the FDR is controlled at level $q$} if
\begin{align*}
\mathbb{E}\left(\frac{v}{r}\right)\leq q.
\end{align*}
\end{Definition}
In our setting, a null hypothesis states that the columns and rows indicated by the vector $x$, respectively $y$, of a tile $(x,y)$ do not stand in any underlying relation. In other words, the hypothesis $H_s^0$ is true when there is no overlap between the underlying model $Y^*{X^*}^\top$ and the outer product of the $s$-th tile $Y_{\cdot s}X_{\cdot s}^\top$. This definition of a null hypothesis might seem too restrictive for some applications. Therefore, we discuss possible relaxations of this requirement in Sec.~\ref{sec:reject}. 
Bearing this in mind, we see that a BMF of rank $r$ corresponds to a joint rejection of $r$ null hypotheses $\{H_1^0,H_2^0,\dots,H_r^0\}$. 
Thus, if the correct rank is $r^*$, any rank $r>r^*$ factorization is likely to state some erroneous rejections of null hypotheses, a.k.a. false discoveries. 

Now, given factor matrices $X\in \{0,1\}^{n\times r}$ and $Y\in\{0,1\}^{m\times r}$, we define a random variable $Z_s$ with domain $\{0,1\}$, which takes the value $1$ if and only if the null hypothesis $H_s^0$ is not to be rejected, i.e., the outer product $Y_{\cdot s}X_{\cdot s}^\top $ covers only (or mostly -- depending on the definition) noise. The FDR of a BMF is therefore computed via
\begin{align}
\mathbb{E}\left(\frac{v}{r}\right) = \frac{1}{r}\sum_{s=1}^r\mathbb{P}(Z_s=1)\;.\label{eq:FDRZ}
\end{align}
\subsection{Tiles in Bernoulli Matrices}
We aim at assessing the probability $\mathbb{P}(Z_s=1)$. Therefore, we need to employ an independence assumption on the noise. 
\begin{Definition}[Bernoulli matrix]
Let $B$ be an $m\times n$ binary matrix. If the entries of $B$ are independent Bernoulli variables, which take the value $1$ with probability $p$ and zero otherwise, i.e., 
\[\mathbb{P}(B_{ji}=1)=p,\ \mathbb{P}(B_{ji}=0)=1-p\;,\]
then $B$ is a \emph{Bernoulli matrix} with parameter $p$.
\end{Definition}
In what follows, we assume that the positive noise matrix $N_+$ is a Bernoulli matrix. If a tile $(X_{\cdot s},Y_{\cdot s})$ approximates noise, the outer product $Y_{\cdot s}X_{\cdot s}^\top$ and the positive noise matrix have some entries in common. The \emph{overlap} is computed by the sum of common $1$ entries
\[|Y_{\cdot s}X_{\cdot s}^\top\circ N_+|=Y_{\cdot s}^\top N_+X_{\cdot s}\;.\]
This quantity is used to determine the density of a tile in the positive noise matrix.
\begin{Definition}[$\delta$-dense]
Let $M$ be an $m\times n$ binary matrix and $\delta\in[0,1]$. We say a tile $(X_{\cdot s},Y_{\cdot s})$ is \emph{$\delta$-dense in $M$}  if
\[Y_{\cdot s}^\top M X_{\cdot s} \geq \delta |X_{\cdot s}||Y_{\cdot s}|\;.\]
\end{Definition}
We will see in the following section, that a Boolean matrix product, which approximates the data matrix well, covers a high proportion of ones in $D$. Therefore, the tiles returned by a BMF are expected to be dense in $D$ ($\delta>0.5$). We explore by the following theorem the probability with that a $\delta$-dense tile of given minimal size exists in a Bernoulli matrix. This gives us an upper bound on the probability $\mathbb{P}(Z_s=1)$ from Eq.~\eqref{eq:FDRZ}, which in turn allows us to bound the FDR.
\begin{theorem}\label{thm:densProb}
Suppose $B$ is an $m\times n$ Bernoulli matrix with parameter $p$, $\delta$ is in $[0,1]$, $1\leq a\leq n$, and $1\leq b\leq m$. 
The probability that a $\delta$-dense tile of size $|x|\geq a$ and $|y|\geq b$ exists is no larger than
\begin{align}\label{eq:densBound}
\binom{n}{a}\binom{m}{b}\exp(-2ab(\delta-p)^2)\;.
\end{align}
\end{theorem}
\begin{proof}
If a $\delta$-dense tile $(x,y)$ exists in $B$, having the size  $(|x|,|y|)\geq (a,b)$, then we can construct a $\delta$-dense sub-tile of exact size $(a,b)$. This follows from the observation that removing the sparsest column/row in $y^\top Bx$ from the tile does not decrease the density. Thus, the probability that a $\delta$-dense tile of size at least $(a,b)$ exists is no larger than the probability that a tile of size $(a,b)$ exists.  

Now, let $(x,y)$ be such a tile with $|x|= a$ and $|y|= b$. The probability that $(x,y)$ is $\delta$-dense in $B$ is equal to
\begin{align*}
\mathbb{P}\left(\frac{y^\top Bx}{|x||y|}\geq \delta\right)
&=\mathbb{P}\left(\left(\frac{1}{ab}\sum_{i,j}x_iy_jB_{ji}\right)-p\geq \delta-p\right)\\
&\leq \exp(-2ab(\delta-p)^2),
\end{align*}
where the inequality follows from Hoeffding's inequality.
An application of the union bound over all possible combinations to place $a$ ones in $x$ and $b$ ones in $y$ yields the statement of the theorem.\hfill$\blacksquare$
\end{proof}
The proof of Theorem~\ref{thm:densProb} indicates that the tightness of Bound~\eqref{eq:densBound} might suffer from the extensive use of the union bound. This originates from the numerous possibilities to select a set of columns and rows of given cardinality. If we expect that rows and columns which are selected by a tile have proportionately many ones in common, we bypass the requirement to take all possible column and row selections into account. To this end, given an $(m\times n)$-dimensional matrix $B$, we assess the value of the function 
\[
\eta(B)= \max_{1\leq i\neq k\leq n} \langle B_{\cdot i},B_{\cdot k}\rangle\;.
\] 
\begin{theorem}
Let $B$ be an $m\times n$ Bernoulli matrix with parameter $p$ and let $\mu>p^2$. The function value of $\eta$ satisfies
$
\eta\left(({1}/\sqrt{m})B\right) \geq \mu
$
with probability no larger than
\begin{align}\label{eq:cohBound}  
\frac{n(n-1)}{2}\exp\left( -\frac{3}{2}m\frac{(\mu-p^2)^2}{2p^2+\mu}\right). 
\end{align}
\end{theorem}
\begin{proof}
Let $B$ be as described above and $1\leq i\neq k\leq n$. The variance of the random variable $B_{ji}B_{jk}$ is
\[\mathbb{E}\left[\left(B_{ji}B_{jk}-p^2\right)^2\right]=p^2(1-p^2)\;.\]
Since the variables $B_{ji}B_{jk}$ are independent for $j\in\{1,\ldots,m\}$, the Bernstein inequality yields 
\begin{align*}
\mathbb{P}(\langle B_{\cdot i},B_{\cdot k}\rangle\geq m\mu) & \leq \exp\left( -\frac{3}{2}m\frac{(\mu-p^2)^2}{2p^2+\mu}\right)\;,
\end{align*}
where we made use of the relations $\mathbb{E}[B_{ji}B_{jk}]=p^2$ and $1-p^2\leq 1$.
The union bound over all possible pairs of distinct rows ($i\neq k$) yields the final result.\ \hfill$\blacksquare$
\end{proof}

If the columns of matrix $B$ are normalized, then the function $\eta(B)$ returns the \emph{coherence} of $B$. The coherence measures how close the column vectors are to an orthogonal system, an extensively studied property in the field of compressed sensing~\cite{foucart2013CS}. If all columns of a matrix are orthogonal to each other, then the coherence is zero. The bound in Eq.~\eqref{eq:cohBound} also implies a bound on the coherence of the matrix $B$. Thus, we refer to Bound~\eqref{eq:cohBound} as the \emph{coherence bound} and to Bound~\eqref{eq:densBound} as the \emph{density bound}. For any given tile, we can now derive two upper bounds on the quantity $\mathbb{P}(Z_s=1)$ from Eq.~\eqref{eq:FDRZ}, and thus control the FDR.
\subsection{Rejecting the Rejection of Null Hypotheses}\label{sec:reject}
How does the density and the coherence bound now help assessing the  probability $\mathbb{P}(Z_s=1)$ from Eq.~\eqref{eq:FDRZ}? Let us reconsider the universal formulation of a null hypothesis, which poses that a tile does not reflect actual relations given by the model $Y^*{X^*}^\top$. First, we relax this definition by counting those tiles which cover only a fraction of the underlying model among the false discoveries as well. Given factor matrices $X$ and $Y$ and and a fraction parameter $t\in[0,1]$, we define the null hypothesis $H_s^0(t)$ to be true if the overlap between the $s$-th tile and the model is smaller than 
\begin{align}\label{eq:nullHyp}
\frac{|Y_{\cdot s}X_{\cdot s}^\top\circ Y^*{X^*}^\top|}{|Y_{\cdot s}X_{\cdot s}^\top|}\leq t.
\end{align}
\begin{corollary}\label{thm:densZ}
Let $D$ be composed as denoted in Eq.~\eqref{eq:BMF0} and let $N_+$ be a Bernoulli matrix with parameter $p$. Given $X\in\{0,1\}^{n\times r}$ and $Y\in\{0,1\}^{m\times r}$ with 
\[Y_{\cdot s}^\top DX_{\cdot s}\geq \delta_s|Y_{\cdot s}||X_{\cdot s}|,\] 
then for $\rho \coloneqq \max\{\delta_s-t-p,0\}$, the probability that  Eq.~\eqref{eq:nullHyp} holds is given by
\begin{align*}
\mathbb{P}(Z_s=1)\leq \binom{n}{|X_{\cdot s}|}\binom{m}{|Y_{\cdot s}|}\exp(-2|X_{\cdot s}||Y_{\cdot s}|\rho^2)
\end{align*}
\end{corollary}
\begin{proof}
We apply the triangle inequality to the decomposition of $D$ given by Eq.~\eqref{eq:BMF0}, yielding
\[|Y_{\cdot s}^\top DX_{\cdot s}|\leq |Y_{\cdot s}X_{\cdot s}^\top\circ Y^*{X^*}^\top|+|Y_{\cdot s}X_{\cdot s}^\top \circ N_+|.\]
Dividing by $|Y_{\cdot s}||X_{\cdot s}|$ and applying Eq.~\eqref{eq:nullHyp} yields that $(X_{\cdot s},Y_{\cdot s})$ is $\delta_s-t$-dense in $N_+$.
The probability for this event is bounded by Theorem~\ref{thm:densProb}.\ \hfill$\blacksquare$
\end{proof}
Similar considerations lead to a false discovery bound based on coherence. Therefore, we define the null hypothesis to hold if 
\begin{align}\label{eq:nullHypCoh}
\eta(Y_{\cdot s}X_{\cdot s}\circ Y^*{X^*}^\top)\leq t.
\end{align}
This restriction affects the tile-wise overlap between the underlying and the computed model more than the definition based on density does. As such, Eq.~\eqref{eq:nullHypCoh} implies that each column of the outer product $Y_{\cdot s}{X_{\cdot s}}^\top$ covers at most $t$ rows of each tile $(X^*_{\cdot \tilde{s}},Y^*_{\cdot \tilde{s}})$ of the underlying model. The probability of a false discovery according to this definition of a null hypothesis is bounded by the following corollary. 
\begin{corollary}\label{thm:cohZ}
Let $D$ be composed as denoted in Eq.~\eqref{eq:BMF0} and let $N_+$ be a Bernoulli matrix with parameter $p$. Given $X\in\{0,1\}^{n\times r}$ and $Y\in\{0,1\}^{m\times r}$ such that 
\[\eta(Y_{\cdot s}X_{\cdot s}^\top \circ D)\geq \mu_s m,\] 
then for $\tilde{\mu} \coloneqq \max\{\mu_s-\nicefrac{t}{m},p^2\}$, the probability that Eq.~\eqref{eq:nullHypCoh} holds is bounded by
\begin{align*}
\mathbb{P}(Z_s=1)\leq \exp\left(-\frac{3}{2}m\frac{(\tilde{\mu}-p^2)^2}{2p^2+\tilde{\mu}}\right)
\end{align*}
\end{corollary}
\begin{proof} 
From the composition of $D$ as denoted in Eq.~\eqref{eq:BMF0} and the definition of $\eta$, computing a maximum, follows that 
\[\eta(Y_{\cdot s}X_{\cdot s}^\top \circ D)\leq \eta(Y_{\cdot s}X_{\cdot s}^\top \circ Y^*_{\cdot s}{X^*_{\cdot s}}^\top)+\eta(Y_{\cdot s}X_{\cdot s}^\top \circ N_+).\]
Applying Eq.~\eqref{eq:nullHypCoh} and $\eta(Y_{\cdot s}X_{\cdot s}^\top \circ D)\geq \mu m$ yields
$\eta(N_+)\geq \mu m-t$.
The probability that this inequality holds is bounded by Theorem~\ref{thm:minCohTile}.\ \hfill$\blacksquare$
\end{proof}
We assume from now on that $t=0$, by what both definitions of the null hypothesis concur. The following results are though easily adapted to a parametrized definition of the null hypothesis.  
\section{Theoretical Comparison of Proposed Bounds}\label{sec:theoryCompare}
The bounds from the previous section supposedly enable a theoretically well-founded approach to select the rank of a BMF. For any factorization, we can now toss all tiles which may just as well have arisen from noise.
However, the tightness of the bounds is the linchpin of the applicability of this scheme. 

Since we do not require a penalization term of the model complexity to determine the correct rank in the FDR controlled scenario, we can choose the most simple objective function. That is, we regard the following optimization problem for a given rank $r$:
\begin{equation}
\begin{aligned}
\min_{X,Y}L(X,Y)= &|D-Y\odot X^\top|\\
\text{s.t. }& X\in \{0,1\}^{n\times r}, Y\in\{0,1\}^{m\times r}
\end{aligned}
\tag{P}\label{eq:P}
\end{equation}
The minimization of the residual sum of absolute values $L(X,Y)$ is not only simple to implement, but this function is also simple enough to let us derive characteristics of its optima with regard to coherence and minimum density of tiles. This enables a theoretic characterization of those tiles which would be tossed by the bounds. Moreover, this contributes to a fundamental understanding of the nature of tiles in a minimizing BMF. Assuming the data is composed as stated in Eq.~\eqref{eq:BMF0}, we explore the circumstances which have to be met such that a tile in the noise matrix contributes to minimizing $L(X,Y)$.
\begin{lemma}\label{thm:minDensOverlap}
Let $X$ and $Y$ be $n\times r$ and $m\times r$ binary matrices and let $\tilde{s}\in\{1,\ldots,r\}$. If $(X,Y)$ is a solution of \eqref{eq:P}, then the density of tile $(x,y)=(X_{\cdot \tilde{s}},Y_{\cdot \tilde{s}})$ is bounded below on the area which is  not covered by any other tile, i.e.,
\begin{align}\label{eq:densOverlap}
\frac{y^\top (D\circ \thickbar{M})x}{y^\top \thickbar{M}x}
\geq \frac{1}{2},
\end{align}
where $M=\left(\bigvee_{s\neq \tilde{s}}Y_{j s}X_{i s}^\top\right)_{ji}$ denotes the boolean product of the factor matrices, excluding the $\tilde{s}$-th tile.
\end{lemma}
\begin{proof}
Let $X,\ Y$, $M$ and $(x,y)$ be as described above. The Boolean product  at position $(j,i)$ is
\begin{align*}
(Y\odot X^\top )_{ji}&= \bigvee_{s=1}^rY_{js}X_{js}= M_{ji} + y_jx_i\thickbar{M_{ji}}.
\end{align*}
The approximation error of a Boolean product is the sum of uncovered ones and covered zeros in the data: 
\begin{align*}
&|D-Y\odot X^\top | \\
=\ &\langle D,\thickbar{Y\odot X^\top }\rangle_F +\langle\thickbar{D},Y\odot X^\top \rangle_F\\
=\ &\langle D,\thickbar{M+ yx^\top \circ \thickbar{M}}\rangle_F +\langle\thickbar{D},M+ yx^\top \circ \thickbar{M}\rangle_F\\
=\ &\langle D,\thickbar{M}\rangle_F - y^\top (D\circ \thickbar{M})x +\langle\thickbar{D},M\rangle_F + y^\top (\thickbar{D}\circ \thickbar{M})x\\
=\ &|D-M| - y^\top (D\circ \thickbar{M})x + y^\top (\thickbar{D}\circ \thickbar{M})x.
\end{align*}
Since $(X,Y)$ minimizes $L$, the gap between the function values $|D-M|-L(X,Y)$ is nonnegative. Hence
\begin{align*}
& y^\top (D\circ \thickbar{M})x - y^\top (\thickbar{D}\circ\thickbar{M})x\\
&= 2y^\top (D\circ \thickbar{M})x - y^\top \thickbar{M}x \geq 0.
\end{align*}
Transforming this inequality yields the final result.\hfill$\blacksquare$
\end{proof}
Note that the proof of Lemma~\ref{thm:minDensOverlap} implies, that the density in Eq.~\eqref{eq:densOverlap} has to be larger than one half, if the objective function incorporates a regularization term on the factor matrices. This could be, e.g., the $\ell1$-norm of the matrices. From Lemma~\ref{thm:minDensOverlap} we now conclude the following property of tiles which are a false discovery.  
\begin{corollary}
Let the matrices $X$ and $Y$ solve \eqref{eq:P}, and let $\tilde{s}\in\{1,\ldots,r\}$. If the tile $(X_{\cdot \tilde{s}},Y_{\cdot \tilde{s}})$ is a false discovery and has no overlap with the remaining tiles, i.e., $(Y_{\cdot s}^\top Y_{\cdot \tilde{s}})(X_{\cdot s}^\top X_{\cdot \tilde{s}})=0$ for $s\neq \tilde{s}$, then $(X_{\cdot \tilde{s}},Y_{\cdot \tilde{s}})$ is $\nicefrac{1}{2}$-dense in $N$.
\end{corollary}
A similar procedure leads to a bound on the coherence.
\begin{lemma} \label{thm:minCohTile}
Let the matrices $X$ and $Y$ solve \eqref{eq:P} and let $s\in\{1,\ldots,r\}$. If $(X_{\cdot s},Y_{\cdot s})$ is $\delta$-dense in $D$, then
\begin{align}
\eta(D) > \delta|Y_{\cdot s}|\frac{\delta|X_{\cdot s}|-1}{|X_{\cdot s}|-1}\label{eq:minCohTile}
\end{align}
\end{lemma}
\begin{proof}
Let $X,Y$ and $s$ be described as above. Denote by $\mathcal{K}=\{i\in\mathcal{I}\mid X_{is}=1\}$ the set of all items indicated by $X_{\cdot s}$. Since the $\ell1$-norm is bounded for a vector $x$ with $a$ nonzero entries by $|x|\leq\sqrt{a}\|x\|$ and since $(X_{\cdot s},Y_{\cdot s})$ is $\delta$-dense, it holds that
\begin{align*}
\|\diag(Y_{\cdot s})DX_{\cdot s}\|^2&\geq \frac{|Y_{\cdot s}^\top DX_{\cdot s}|^2}{|Y_{\cdot s}|}\geq \delta^2|Y_{\cdot s}||X_{\cdot s}|^2, 
\end{align*}
where $\diag(y)$ denotes the diagonal matrix whose nonzero entries equal $y$. The norm above is equal to
\begin{align*}
\|\diag(Y_{\cdot s})DX_{\cdot s}\|^2& = X_{\cdot s}^\top D^\top\diag(Y_{\cdot s})DX_{\cdot s}\\
&=\sum_{i,k\in\mathcal{K}} D_{\cdot i}^\top \diag(Y_{\cdot s})D_{\cdot k}\\
&= Y_{\cdot s}^\top DX_{\cdot s}+\sum_{i\neq k\in\mathcal{K}}D_{\cdot i}^\top \diag(Y_{\cdot s})D_{\cdot k}.
\end{align*}
Combining both (in)equalities above yields
\begin{align*}
\sum_{i\neq k\in\mathcal{K}}D_{\cdot i}^\top \diag(Y_{\cdot s})D_{\cdot k}& \geq Y_{\cdot s}^\top DX_{\cdot s}\left(\frac{Y_{\cdot s}^\top DX_{\cdot s}}{|Y_{\cdot s}|}-1\right)\\
&\geq \delta|Y_{\cdot s}||X_{\cdot s}|\left(\delta|X_{\cdot s}|-1\right).
\end{align*}
According to the pigeonhole principle, indices $i\neq k$ exist, $i,k\in\mathcal{K}$ such that
\begin{displaymath}
~\qquad\qquad\quad\langle D_{\cdot i}, D_{\cdot k}\rangle
> \delta|Y_{\cdot s}|\frac{\delta|X_{\cdot s}|-1}{|X_{\cdot s}|-1}.\qquad\qquad\quad\blacksquare
\end{displaymath}
\end{proof}
If we assume that a tile $(X_{\cdot s},Y_{\cdot s})$ is a false discovery from an optimal solution $(X,Y)$ of \eqref{eq:P}, then Bound~\eqref{eq:minCohTile} applies to $N_+$.
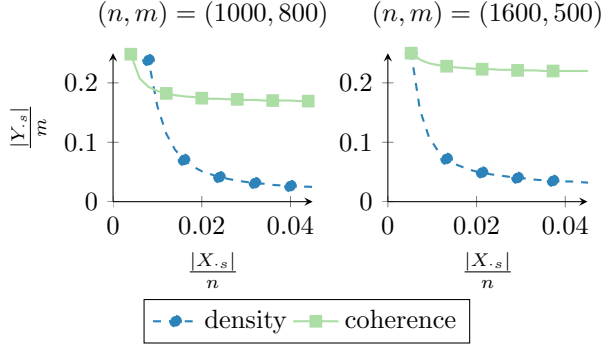
\begin{figure}
\centering
\pgfplotsset{
trustStyle/.style={ctrust,mark options ={ctrust},mark repeat={4}, thick, error bars/.cd,y dir = both, y explicit},
densStyle/.style={cdens,dashed,mark options ={cdens},mark repeat={4}, thick, error bars/.cd,y dir = both, y explicit},
}

\begin{tikzpicture}
    \begin{groupplot}[group style={group size= 2 by 1, vertical sep=0.6cm},
    	height=.42\columnwidth,
    	width=.5\columnwidth,
        axis lines = left,
        scaled ticks=false,tick label style={/pgf/number format/fixed}]
        \nextgroupplot[
        	title={$(n,m)=(1000, 800)$},
        	ylabel={$\frac{|Y_{\cdot s}|}{m}$},
            xlabel={$\frac{|X_{\cdot s}|}{n}$},
            xmin=0,xmax=0.045, ymin=0, ymax=0.25,
            legend columns =-1,
        	legend to name=zelda]
        	\addplot+[densStyle]  plot coordinates {
 (0.002,0.80525)
 (0.004,0.56325)
 (0.006,0.36825)
 (0.008,0.23825)
 (0.01,0.16025) 
 (0.012,0.11425)
 (0.014,0.08725)
 (0.016,0.07025)
 (0.018,0.05925)
 (0.02,0.05125) 
 (0.022,0.04625)
 (0.024,0.04125)
 (0.026,0.03825)
 (0.028,0.03525)
 (0.03,0.03325) 
 (0.032,0.03125)
 (0.034,0.03025)
 (0.036,0.02825)
 (0.038,0.02725)
 (0.04,0.02625) 
 (0.042,0.02525)
 (0.044,0.02525)
 (0.046,0.02425)
 (0.048,0.02325)
 (0.05,0.02325) 
    };
    \addlegendentry{density};
    \addplot+[trustStyle]  plot coordinates {
 (0.002,1)
 (0.004,0.24825)
 (0.006,0.20725)
 (0.008,0.19325)
 (0.01,0.18625) 
 (0.012,0.18225)
 (0.014,0.17925)
 (0.016,0.17725)
 (0.018,0.17625)
 (0.02,0.17425) 
 (0.022,0.17325)
 (0.024,0.17325)
 (0.026,0.17225)
 (0.028,0.17225)
 (0.03,0.17125) 
 (0.032,0.17125)
 (0.034,0.17025)
 (0.036,0.17025)
 (0.038,0.17025)
 (0.04,0.17025) 
 (0.042,0.16925)
 (0.044,0.16925)
 (0.046,0.16925)
 (0.048,0.16925)
 (0.05,0.16925) 
    };\addlegendentry{coherence};
	\nextgroupplot[title={$(n,m)=(1600, 500)$},
    	xmin=0,xmax=0.045, ymin=0, ymax=0.25,
    	xlabel={$\frac{|X_{\cdot s}|}{n}$}]
                \addplot+[densStyle]  plot coordinates {
 (0.00125,0.812)
 (0.00325,0.459)
 (0.00525,0.25) 
 (0.00725,0.154)
 (0.00925,0.109)
 (0.01125,0.086)
 (0.01325,0.072)
 (0.01525,0.063)
 (0.01725,0.057)
 (0.01925,0.052)
 (0.02125,0.049)
 (0.02325,0.046)
 (0.02525,0.044)
 (0.02725,0.042)
 (0.02925,0.04) 
 (0.03125,0.039)
 (0.03325,0.037)
 (0.03525,0.036)
 (0.03725,0.035)
 (0.03925,0.034)
 (0.04125,0.034)
 (0.04325,0.033)
 (0.04525,0.032)
 (0.04725,0.032)
 (0.04925,0.031)
    };
            	\addplot+[trustStyle]  plot coordinates {
 (0.00125,1)
 (0.00325,0.284)
 (0.00525,0.25) 
 (0.00725,0.239)
 (0.00925,0.233)
 (0.01125,0.23) 
 (0.01325,0.228)
 (0.01525,0.226)
 (0.01725,0.225)
 (0.01925,0.224)
 (0.02125,0.223)
 (0.02325,0.223)
 (0.02525,0.222)
 (0.02725,0.222)
 (0.02925,0.221)
 (0.03125,0.221)
 (0.03325,0.221)
 (0.03525,0.22) 
 (0.03725,0.22) 
 (0.03925,0.22) 
 (0.04125,0.22) 
 (0.04325,0.22) 
 (0.04525,0.22) 
 (0.04725,0.219)
 (0.04925,0.219)
    };
    \end{groupplot}
\end{tikzpicture}\\

\pgfplotslegendfromname{zelda}
\caption{Minimum relative size $\nicefrac{|Y_{\cdot s}|}{m}$, depending on $\nicefrac{|X_{\cdot s}|}{n}$, for which the $\mathbb{P}(Z_s=1)\leq 0.01$, based on density (blue) and coherence (green).}
\label{fig:theory}
\end{figure}

These results enable a theoretic comparison of the bounds based on coherence and density. Fig.~\ref{fig:theory} contrasts the two bounds for two settings of dimensions. The plot on the left refers to almost square dimensions $(n,m)=(1000,800)$ and the one on the right to more imbalanced dimensions $(n,m)=(1600,500)$. Let $(X,Y)$ be a solution of $\eqref{eq:P}$ and assume that the positive noise matrix is a Bernoulli matrix with probability $p=0.1$. We plot the minimum relative size $\nicefrac{|Y_{\cdot s}|}{m}$ against the relative size $\nicefrac{|X_{\cdot s}|}{n}$ such that the probability $\mathbb{P}(Z_s=1)\leq 0.01$. The blue curve displays the minimum tile size, assessing the false discovery probability by Corollary~\ref{thm:densZ}, while green refers to Corollary~\ref{thm:cohZ}. Thereby, we assume that the tile is $\nicefrac{1}{2}$-dense in $N_+$ and the value $\eta(N_+)$ is bounded by Inequality~\eqref{eq:minCohTile}. Fig.~\ref{fig:theory} indicates that under the given circumstances the coherence provides a more loose bound than the density. The difference between the required sizes is larger, if the dimensions are disproportionate, which suggests that more tiles are rejected as potential false discoveries by the coherence bound, in particular for wide or tall data matrices.     
\section{Algorithmic Integration of FDR Control}\label{sec:alg}
The false discovery bounds might be applied as a postprocessing step to any BMF. Here, we also establish the use of these bounds to directly estimate the rank.
The framework \textsc{PAL-Tiling}~\cite{Hess/etal/2017a} is well suited for that matter. \textsc{PAL-Tiling} applies recent results from non-convex optimization theory to minimize a relaxed objective function $F$ for matrices with entries between zero and one. A regularizing function $\phi$ penalizes non-binary values in the factor matrices. A rounding procedure at the end of each optimization decides over the estimated factorization rank. In this step, we naturally integrate a check of the provided false discovery bounds. 

The numerical optimization of the relaxed objective is performed by \emph{Proximal Alternating Linearized Minimization} (PALM)~\cite{palm14}. This scheme invokes alternating \emph{proximal mappings} with respect to $\phi$ from the gradient descent update with respect to $F$ (cf.\@ lines~\ref{alg:proxX} and \ref{alg:proxY} in Algorithm~\ref{alg:trustPal}).
Since we intend to solve Problem~\eqref{eq:P}, a suitable smooth relaxed objective is the residual sum of squares $F(X,Y)=\nicefrac{1}{2}\|D-YX^T\|^2$, as discussed in~\cite{Hess/etal/2017a}. The step sizes of the gradient descent updates are computed by the Lipschitz moduli of partial gradients (cf.\@ lines~\ref{alg:stepX} and \ref{alg:stepY}). 
PALM generates a nonincreasing sequence of function values $F(X,Y)+\phi(X)+\phi(Y)$ which converges to a critical point.
\begin{algorithm}[t]
\caption{PAL-Tiling with FDR Control}
\begin{algorithmic}[1]
  \Function{TrustPal}{$D,\hat{p};\Delta_r=10,q = 0.01$}
    \State $(X_K,Y_K)\gets (\emptyset, \emptyset)$
    \For {$r\in\{\Delta_r,2\Delta_r,3\Delta_r,\ldots\}$}
    \State $(X_0,Y_0) \gets $\Call{IncreaseRank}{$X_K, Y_K,\Delta_r$} 
    \For {$k = 0,1,\ldots$}\label{alg:optStart} \Comment{Select stop criterion}\label{alg:optStart}
    	\State $\alpha_k^{-1} \gets \|Y_kY_k^\top \|$\label{alg:stepX}
        \State $X_{k+1} \gets \prox_{\alpha_k\phi}\left(X_k-\alpha_k\nabla_XF(X_k,Y_k)\right)$\label{alg:proxX}
       	\State $\beta_k^{-1} \gets\|X_{k+1}X_{k+1}^\top \|$\label{alg:stepY}
        \State $Y_{k+1} \gets \prox_{\beta_k\phi}\left(Y_k-\beta_k\nabla_YF(X_{k+1},Y_k)\right)$\label{alg:proxY}
    \EndFor\label{alg:optEnd}
    \State $(X,Y)\gets \Call{RoundFDR}{L,X_k,Y_k,\hat{p},q}$ \label{alg:round}
    \If {$\Call{RankGap}{X,Y,r}$}\label{alg:rankStart}
    	\textbf{return} $(X,Y)$
    \EndIf\label{alg:rankEnd}
    \EndFor
  \EndFunction
\end{algorithmic}
\label{alg:trustPal}
\end{algorithm}

Given the data matrix $D$, the estimated noise probability $\hat{p}$, the rank increment $\Delta_r$ (default value 10) and the FDR control level $q$ (default value 0.01), we propose the method \textsc{TrustPal} as sketched in Algorithm~\ref{alg:trustPal}. The rank of the initially empty factor matrices is iteratively incremented. With every rank increment, $\Delta_r$ random columns are appended to the current factor matrices and the numerical optimization by PALM is performed until a selected stop criterion, e.g., maximum number if iterations or minimum function decrease, is met. 

The function $\textsc{RoundFDR}$ (cf.\@ line~\ref{alg:round}) computes for every pair of rounding thresholds $(\tau_x,\tau_y)\in\{0,0.05,0.1,\ldots,1\}$ the candidate binary matrices 
\[X_{is}=\lceil {X_k}_{is}-\tau_x\rceil,\quad Y_{js}=\lceil {Y_k}_{js}-\tau_y\rceil.\]
If the probability of a false discovery is not bounded above by $q$ for tile $(X_{\cdot s},Y_{\cdot s})$, i.e., Corollary~\ref{thm:densZ} or \ref{thm:cohZ} does not yield $\mathbb{P}(Z_s=1)\leq q$, then the tile is removed from the factorization:
$(X_{\cdot s},Y_{\cdot s})=(\mathbf{0},\mathbf{0})$.
Thereby, Corollary~\ref{thm:cohZ} is also applied to the transposed data factorization for a symmetric test of the coherence bound. The function $\textsc{RoundFDR}$ returns the matrices which minimize the residual sum of absolute squares $L(X,Y)$.
Finally, if the gap between the possible rank $r$ and the rank of the returned factorization is larger than a specified value, the current solution is returned.
\section{Experiments}\label{sec:experiments}
Our experimental evaluation serves the assessment of provided bounds in practical applications. Although the theoretical properties of minimizing factorizations yield satisfactory bounds on the size of a tile (cf.\@ Fig.~\ref{fig:theory}), in practice no feasible existing algorithm can guarantee to return optimal solutions of Problem~\eqref{eq:P}. 
In addition, we show that the estimation of the actual noise probability is not critical in practice.

The implementation of \textsc{TrustPal} follows the highly parallel implementation on graphics processing units (GPU) from the framework \textsc{PAL-Tiling}. All experiments are executed on a GPU with 2688 arithmetic cores and 6GiB GDDR5 memory. The source code of \textsc{TrustPal}, together with Julia scripts to generate data and to compare proposed bounds, is provided\footnote{\url{http://sfb876.tu-dortmund.de/trustpal}}.

We compare the two variants of \textsc{TrustPal}, employing the  bounds based on  density or coherence to determine the rank, and the performance of the algorithm \textsc{Primp}. \textsc{Primp} is an instance of \textsc{PAL-Tiling} which minimizes a complicated, highly non-convex and non-smooth function. Yet, this function enables the desirable ability of \textsc{Primp} to correctly estimate the rank and to derive meaningful factorizations. In contrast, the other instance of \textsc{PAL-Tiling}, \textsc{PanPal}, minimizes a simple $\ell1$ regularization of the residual sum of absolute errors. \textsc{PanPal} displays a strong tendency to underestimate the rank. Nevertheless, \textsc{PanPal} is able to yield more accurate results than \textsc{Primp} in some settings~\cite{Hess/etal/2017a}. Whenever that is the case, we display the results for \textsc{PanPal} in the following plots. 
Concerning \textsc{TrustPal}, we employ a stopping criterion of maximum 2000 optimization iterations or a minimum function decrease of 0.0001. 
\subsection{Experiments on Synthetic Data}
We generate $(1600\times 500)$ and $(1000\times 800)$ dimensional datasets according to the scheme established in \cite{Hess/etal/2017a}. Given dimensions $n,m$, rank $r^\star$, maximum relative tile size $d\in[0,1]$ and positive and negative noise parameters $p_\pm$, a factorization of rank $r^\star$ is generated by uniformly randomly drawing each tile $(X^*_{\cdot s},Y^*_{\cdot s})$ from all tiles of size $|X^*_{\cdot s}|\in[0.01n,dn]$ and $|Y^*_{\cdot s}|\in[0.01m,dm]$. Finally, each entry $(Y^*{X^*}^\top)_{ji}=0$ is flipped to one with positive noise probability $p_+$ and every bit $(Y^* {X^*}^\top)_{ji}=1$ is flipped to zero with negative noise probability $p_-$. If not stated otherwise, the default settings $p_+=p_-=0.1$, $r^*=25$ and $d=0.1$ apply. 

We compare the computed models against the planted structure by an adaptation of the micro-averaged $F$-measure.
Details on how to compute the $F$-measure for this kind of data can be found in \cite{Hess/etal/2017a}. 
The $F$-measure obtains values between zero and one. The closer it approaches one, the more similar computed and planted factorizations are.
\begin{figure}[t]
\centering
\pgfplotsset{
cohStyle/.style={ctrust,mark options ={ctrust},mark repeat={4}, thick, error bars/.cd,y dir = both, y explicit},
trustStyle/.style={cdens,dashed,mark options ={cdens},mark repeat={4}, thick, error bars/.cd,y dir = both, y explicit},
primpStyle/.style={cPrimp,dashed,mark=triangle*,mark options ={cPrimp},mark repeat={3}, thick, error bars/.cd,y dir = both, y explicit},
}

\begin{tikzpicture}
    \begin{groupplot}[group style={group size= 2 by 2, vertical sep=0.6cm},
    	height=.45\columnwidth,
    	width=.5\columnwidth,
        legend columns =-1,
        axis lines = left]
        \nextgroupplot[ylabel={$F$},
        	title={$(n,m)=(1000,800)$},
        	xmin=0,xmax=25.5, ymin=0.4, ymax=1.1,
        	legend to name=zelda]
        	\addplot+[trustStyle]  table[x=x,y=y,y error=std] {BMFDR_DENS_n1000m800F.dat};
            \addlegendentry{\textsc{TrustPal} dens};
        	\addplot+[cohStyle]  table[x=x,y=y,y error=std] {BMFDR_COH_n1000m800F.dat};
            \addlegendentry{\textsc{TrustPal} coh};
            \addplot+[primpStyle]  table[x=x,y=y,y error=std] {Primp_n1000m800F.dat};
            \addlegendentry{\textsc{Primp}};
        \nextgroupplot[title={$(n,m)=(1600,500)$},
        	xmin=0,xmax=25.5, ymin=0.4, ymax=1.1]
                \addplot+[trustStyle]  table[x=x,y=y,y error=std] {BMFDR_DENS_n1600m500F.dat};
        	\addplot+[cohStyle]  table[x=x,y=y,y error=std] {BMFDR_COH_n1600m500F.dat};
            \addplot+[primpStyle]  table[x=x,y=y,y error=std] {Primp_n1600m500F.dat};
        \nextgroupplot[xlabel={$p_{\pm}\ [\%]$},
			ylabel={$r$},
			xmax=25.5, xmin=0, ymin=0, ymax=34.2]
                \addplot+[trustStyle]  table[x=x,y=y,y error=std] {BMFDR_DENS_n1000m800R.dat};
        	\addplot+[cohStyle]  table[x=x,y=y,y error=std] {BMFDR_COH_n1000m800R.dat};
            \addplot+[primpStyle]  table[x=x,y=y,y error=std] {Primp_n1000m800R.dat};
       \nextgroupplot[xlabel={$p_\pm\ [\%]$},xmin=0,xmax=25.5, ymin=0, ymax=34.2]
                \addplot+[trustStyle]  table[x=x,y=y,y error=std] {BMFDR_DENS_n1600m500R.dat};
        	\addplot+[cohStyle]  table[x=x,y=y,y error=std] {BMFDR_COH_n1600m500R.dat};
            \addplot+[primpStyle]  table[x=x,y=y,y error=std] {Primp_n1600m500R.dat};
    \end{groupplot} 
\end{tikzpicture}\\

\pgfplotslegendfromname{zelda}
\caption{Variation of uniform noise for $1600\times 500$ and $1000\times 800$ dimensional data. Comparison of $F$-measures (the higher the better) and the estimated rank of the calculated tiling (the closer to 25 the better) for varying levels of noise indicated on the x-axis.}
\label{fig:noise}
\end{figure}
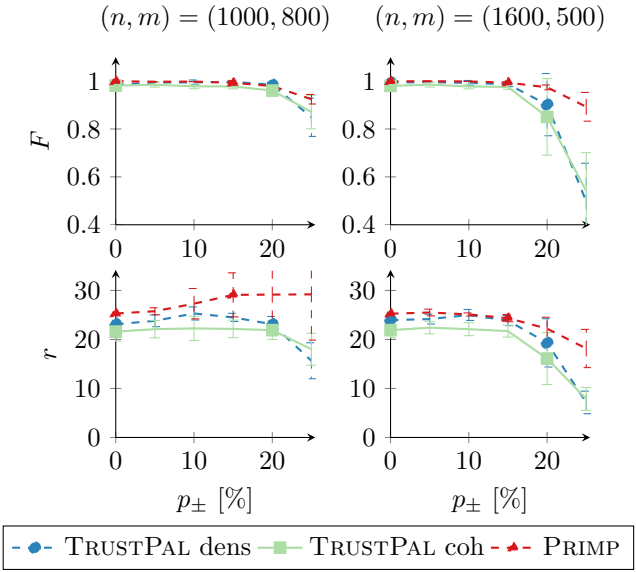

Fig.~\ref{fig:noise} displays the performance of the density and coherence approach of \textsc{TrustPal} with \textsc{Primp}. The noise probability varies between $p_\pm\in\{0,0.05,\ldots,0.25\}$. While the plots on the left show aggregated results over 20 almost square $(1000\times 800)$ matrices, the plots on the right refer to 20 more imbalanced datasets with dimension $(1600\times 500)$. The input parameter of \textsc{TrustPal}, the estimated noise probability, is consistently set to $10\%$. The plots show that both probability bounds yield similar results. In particular, if the noise percentage exceeds the estimated noise probability, no more than the actually planted tiles are discovered. An overestimation of the rank, as happening with \textsc{Primp} on more square matrices, is prevented. 

\begin{table}
	\centering
    \resizebox{\columnwidth}{!}{%
	\begin{tabular}{clrr}\toprule
Vary & Algorithm & $F$ & $r-r^*$  \\ \midrule
\multirow{3}{*}{\rotatebox{90}{$p_+$}  } 
&\textsc{TrustPal} Dens & 0.99 $\pm$ 0.015 & -0.39 $\pm$ 1.65 \\
 & \textsc{TrustPal} Coh & $0.98\pm0.023$ & $-1.77\pm1.57$\\
 & \textsc{Primp} & 0.99 $\pm$ 0.004 & 1.21 $\pm$ 2.89\\
 \midrule
\multirow{3}{*}{\rotatebox{90}{ $r^\star$ }  }  
&\textsc{TrustPal} Dens & 0.98 $\pm$ 0.050 & -0.475 $\pm$ 2.54 \\
 & \textsc{TrustPal} Coh & $0.96\pm0.069$ & $-2.875\pm3.17$\\
 & \textsc{Primp} & 0.98 $\pm$ 0.074 & 0.975 $\pm$ 2.14\\
 \bottomrule
\end{tabular}
}
\caption{Average $F$-measure and difference between computed and planted rank $r-r^\star$ for varied positive noise and planted rank. For each setting the average value is computed over all  dimension variations.}
\label{tbl:avgMeas}
\end{table}

We state the average measures over all variations of the positive noise $p_+\in\{5,10,\ldots,25\}$ and the rank $r^*\in\{5,10,\ldots,45\}$ in Table~\ref{tbl:avgMeas}. For every parameter variation, we generate 8 matrices, 4 for each dimension setting $(n,m)\in\{(1000,800),(1600,500)\}$. We see that all algorithms consistently gain high $F$-values and an average deviation of the rank which is close to zero. Yet, \textsc{Primp}s average rank deviates to a positive amount and \textsc{TrustPal} rather underestimates the rank. Note, that an overestimation of the rank does not necessarily imply a false discovery; planted tiles might be split.
\begin{figure}[t]
\centering
\pgfplotsset{
cohStyle/.style={ctrust,mark options ={ctrust},mark repeat={4}, thick, error bars/.cd,y dir = both, y explicit},
trustStyle/.style={cdens,dashed,mark options ={cdens},mark repeat={4}, thick, error bars/.cd,y dir = both, y explicit},
primpStyle/.style={cPrimp,dashed,mark=triangle*,mark options ={cPrimp},mark repeat={3}, thick, error bars/.cd,y dir = both, y explicit},
panStyle/.style={cPan,dashed,mark options ={cPan},mark repeat={3}, thick, error bars/.cd,y dir = both, y explicit}
}

\begin{tikzpicture}
    \begin{groupplot}[group style={group size= 2 by 1, vertical sep=0.6cm},
    	height=.45\columnwidth,
    	width=.5\columnwidth,
        legend columns =2,
        axis lines = left]
        \nextgroupplot[ylabel={$F$}, xlabel={$d$},
        	xmax=0.35, ymin=0.4, ymax=1.1,
        	legend to name=zelda]
        	\addplot+[trustStyle]  table[x=x,y=y,y error=std] {BMFDR_DENSF.dat};
            \addlegendentry{\textsc{TrustPal} dens};
        	\addplot+[cohStyle]  table[x=x,y=y,y error=std] {BMFDR_COHF.dat};
            \addlegendentry{\textsc{TrustPal} coh};
            \addplot+[primpStyle]  table[x=x,y=y,y error=std] {PrimpF.dat};
            \addlegendentry{\textsc{Primp}};
            \addplot+[panStyle]  table[x=x,y=y,y error=std] {PanPalF.dat};
            \addlegendentry{\textsc{PanPal}};
        \nextgroupplot[xlabel={$d$}, ylabel={$r$},
			xmax=0.35, ymin=0, ymax=40.2]
                \addplot+[trustStyle]  table[x=x,y=y,y error=std] {BMFDR_DENSR.dat};
        	\addplot+[cohStyle]  table[x=x,y=y,y error=std] {BMFDR_COHR.dat};
            \addplot+[primpStyle]  table[x=x,y=y,y error=std] {PrimpR.dat};
            \addplot+[panStyle]  table[x=x,y=y,y error=std] {PanPalR.dat};
    \end{groupplot} 
\end{tikzpicture}\\

\pgfplotslegendfromname{zelda}
\caption{Variation of density and overlap influencing parameter $d\in[0.1,\ldots,0.3]$. Comparison of $F$-measures (the higher the better) and the estimated rank (the closer to 25 the better) for uniform noise of $p=10\%$.}
\label{fig:density}
\end{figure}
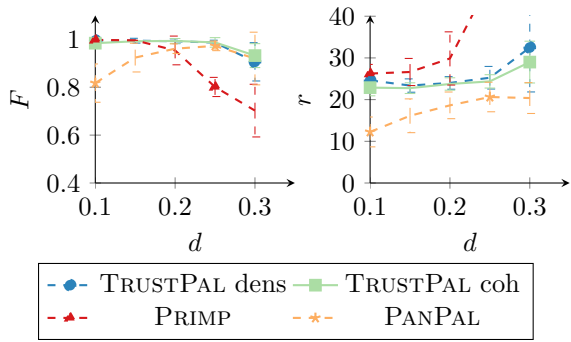

In Fig.~\ref{fig:density} we plot the $F$-measure and the computed rank against the parameter $d$, which bounds the maximum size of a tile. Here, we add a comparison to the algorithm \textsc{PanPal}, whose tendency to underestimate the rank comes in handy for more dense matrices, when $d$ is larger than $0.2$.  We see that \textsc{TrustPal} is able to find the right balance and obtains the highest $F$-measure over all variations of $d$.  

In total, the experimental evaluation suggests that the estimation of the noise probability is not critical in practice and that both bounds are suitable to yield accurate rank estimations under false discovery control.
\subsection{MovieLens Experiments}
Comparing the performance of algorithms in terms of the false discovery rate proves difficult for real-world datasets. Yet, recommendation data at least provides information about negative reviews. Here, we explore the performance of algorithms on the binarized MovieLens1M\footnote{\url{http://grouplens.org/datasets/movielens/1m/}} dataset. Setting  $D_{ji}=1$ iff user $j$ recommends movie $i$ with more than three of five stars, we obtain after pruning of less active users and less often reviewed movies a $(3329\times 3015)$ data matrix with density $4.99\%$. We compare the output of \textsc{TrustPal} for various noise probability estimations $\hat{p}$ to the results from \textsc{Primp} and \textsc{PanPal}.
\begin{table}
\centering
\resizebox{\columnwidth}{!}{%
	\begin{tabular}{ccrrrr}\toprule
     Algorithm & Bound& $\hat{p}$& $r$ & $L(X,Y)$ & Wrong rec.\\ \midrule
\textsc{TrustPal} & coh & 0.1	&23	&80.68	&2.43\\
 & & 0.05	&25	&80.52	&2.38\\
 & & 0.01	&25	&80.26	&2.44\\
 & dens & 0.1	&26	&81.59	&2.11\\
 & & 0.05	&35	&79.54	&2.43\\
 & & 0.01	&25	&78.22	&2.72\\
\textsc{Primp} &-&& 78 & 88.59 & 2.38\\
\textsc{PanPal} &-&& 15 & 94.05 & 2.23\\
\bottomrule
\end{tabular}
}
\caption{Comparison of \textsc{TrustPal} for given noise probabilities $\hat{p}$, \textsc{Primp} and \textsc{PanPal} on the MovieLens dataset. Denoted are the rank $r$, approximation error $L$ and the percentage of traceable wrong recommendations, i.e., user-movie recommendations corresponding to bad reviews ($<2.5$ of five stars).}
\label{tbl:movielens}
\end{table}

Table~\ref{tbl:movielens} summarizes the results. We try small noise probabilities $\hat{p}\in\{0.01,0.05,0.1\}$ as we not often expect that users give a positive rating of a movie they do not actually like. We observe again that a variation of the estimated noise probability does not make much of a difference. The rank tends to increase with decreasing estimated noise and consequently does the approximation error decrease and the amount of wrong recommendations increase. The estimated ranks of \textsc{TrustPal} are close to $25$, which differs notably from the rank of $78$ from \textsc{Primp}. Still, \textsc{TrustPal} achieves lowest approximation error.
\section{Related Work}\label{sec:related}
FDR control in unsupervised settings is basically unexplored. One notable approach is scan clustering~\cite{Pacifico/etal/2007a}, focusing on one- or two-dimensional spatial density clustering. The authors control the area of discovered clusters by the FDR, addressing Gaussian processes in continuous data. Thus, this approach cannot be applied to binary or discrete data in general.  

In the pattern mining literature, a standard framework for handling false discoveries is Webb's Significant Pattern Mining \cite{2007Webb}.  It assesses individual patterns, handling the pattern explosion problem by Bonferroni-like corrections on the significance level. The Significant Pattern Mining framework can work with any null hypothesis to be tested on patterns.  

One considerable approach that works in this setting is statistical significant pattern mining via permutation testing (see \cite{Lopez/etal/2015a} and references therein). The major difference to our scenario is the supervision of the mining procedure. More precisely, patterns are annotated by class labels, and the task is to identify those patterns which appear significantly more often in one class than in the other class(es). State-of-the-art approaches rely on (variants of) Westfall-Young permutation based hypothesis testing. 

In a similar line of research, namely statistical emerging pattern mining \cite{Komiyama/etal/2017a}, patterns from different sources (e.g., databases) are considered. The goal is to find patterns which appear significantly more often in one database than in another. Multiple hypothesis testing is applied to control the FDR, and to provide other statistical guarantees. 

A method designed to test one specific null hypothesis on supervised pattern mining results (such as Subgroup Discovery and Exceptional Model Mining) is DFD Validation \cite{2011Duivesteijn}.  In what essentially boils down to a permutation test, a Distribution of artificial False Discoveries (DFD) is generated.  Subgroups resulting from the actual supervised local pattern mining run are then accepted only if they refute the null hypothesis that they are generated by the DFD.  This provides evidence that the subgroups are deemed interesting by more than solely random effects, but the method is specific to the supervised local pattern mining setting.

All approaches make heavy use of the fact that data comes from multiple classes or sources and are not easily transferred to the unsupervised setting. 
\section{Conclusions}\label{sec:conclusion}
We introduce  a method to control the false discovery rate in Boolean matrix factorization and prove two bounds to estimate the probability that a tile minimizes the objective by covering noise.
A theoretical comparison of our bounds characterizes the tiles which are regarded as false discoveries (cf.\@ Fig~\ref{fig:theory}). 
We explain how FDR control can be integrated into existing algorithms---this improves the theoretical properties of algorithms and takes away the need to regularize the model complexity. An empirical study on synthetic and real-world data demonstrates its practical utility.

In conclusion, FDR control takes the concern about too noisy results off the researcher's hand. The remaining question is how to derive tiles which approach the underlying model best, e.g., which are not split? In this respect, the suitable application of regularizers is still important.
Another arising question is if we can incorporate other noise distributions or how we can test if the noise is, e.g., actually Bernoulli distributed. Multiple avenues of research are opened now.
\subsection*{Acknowledgments}
Part of the work on this paper has been supported by Deutsche Forschungsgemeinschaft (DFG) within the Collaborative Research Center SFB 876 ``Providing Information by Resource-Constrained Analysis'', projects C1 and A1
\url{http://sfb876.tu-dortmund.de}.
\bibliographystyle{abbrv}

\end{document}